%% file: LearningwithaStrongAdversary.tex
\documentclass{article} % For LaTeX2e
\usepackage{iclr2016_conference,times}
\usepackage{hyperref}
\usepackage{url}
\usepackage{amsthm}
\usepackage{amsmath}
\usepackage{amssymb}
\usepackage{algorithm}
\usepackage{algorithmic}
\usepackage{multirow}
\usepackage{array}
\usepackage{graphicx}

\title{Learning with a Strong Adversary}

\author{Ruitong Huang, Bing Xu, Dale Schuurmans and Csaba Szepesv\'ari \\
Department of Computer Science\\
University of Alberta\\
Edmonton, AB T6G 2E8, Canada \\
\texttt{\{ruitong,bx3,daes,szepesva\}@ualberta.ca} \\
}

\newcommand\numberthis{\addtocounter{equation}{1}\tag{\theequation}}

\newcommand{\uZ}{\underline{Z}}
\newcommand{\cU}{\mathcal{U}}
\newcommand{\cV}{\mathcal{V}}
\newcommand{\cN}{\mathcal{N}}
\newcommand{\cW}{\mathcal{W}}
\newcommand{\cX}{\mathcal{X}}
\newcommand{\mI}{\mathbb{I}}
\newcommand{\real}{\mathbb{{R}}}
\newtheorem{prop}{Proposition}
\newtheorem{remark}{Remark}

\begin{document}

\maketitle

\begin{abstract}
The robustness of neural networks to intended perturbations has recently 
attracted significant attention. 
In this paper, we propose a new method, 
\emph{learning with a strong adversary},
that learns robust classifiers from supervised data
by generating adversarial examples as an intermediate step. 
A new and simple way of finding adversarial examples is presented 
that is empirically %more efficient 
stronger
than existing approaches
in terms of the accuracy reduction as a function of perturbation magnitude.
Experimental results demonstrate that resulting learning method
greatly improves the robustness of the classification models produced.
\end{abstract}

\section{Introduction}
Deep Neural Network (DNN) models have recently demonstrated impressive learning 
results in many visual and speech classification problems 
\citep{krizhevsky2012imagenet,hinton2012deep}.
One reason for this success is believed to be the expressive capacity of
deep network architectures.
Even though classifiers are typically evaluated by their 
misclassification rate,
%is the most common performance metric used to evaluate classifiers, 
robustness is also a highly desirable property:
intuitively, it is desirable for a classifier to be `smooth'
in the sense that a small perturbation of its input should not change its
predictions significantly.
An intriguing recent discovery is that DNN models 
do not typically possess such a robustness property 
\citep{szegedy2013intriguing}.
An otherwise highly accurate DNN model can be fooled into misclassifying
typical data points by introducing a human-indistinguishable perturbation 
of the original inputs.
We call such a perturbed data set `adversarial examples'.  
An even more curious fact is that the same set of such adversarial examples 
is also misclassified by a diverse set of classification models, such as KNN, Boosting Tree,
even if they are trained with different architectures and 
different hyperparameters. 

Since the appearance of \citet{szegedy2013intriguing},
increasing attention has been paid to the curious phenomenon of
`adversarial perturbation' in the deep learning community;
see, for example
\citep{goodfellow2014explaining,fawzi2015analysis,miyato2015distributional,nokland2015improving,tabacof2015exploring}. 
\citet{goodfellow2014explaining} suggest that one reason for the
detrimental effect of adversarial examples lies in the implicit
linearity of the classification models in high dimensional spaces.
Additional exploration by \citet{tabacof2015exploring}
has demonstrated that, for image classification problems,
adversarial images inhabit large "adversarial pockets" in the pixel space. 
Based on these observations, 
different ways of finding adversarial examples have been proposed, 
among which the most relevant to our study is that of
\citep{goodfellow2014explaining},
where a linear approximation is used to obviate the need for 
any auxiliary optimization problem to be solved.
In this paper, we further investigate the role of adversarial training
on classifier robustness
and propose a simple new approach to finding 'stronger' adversarial examples. 
Experimental results suggest that the proposed method is more %efficient 
effective
than previous approaches
in the sense that the resulting DNN classifiers obtain
worse performance under the same magnitude of perturbation.
%We show that learning with an adversary still the convexness of logistic regression, thus the optimal solution can be solve efficiently and globally.
%As has been discussed in \citep{goodfellow2014explaining}, learning with an adversary behaves like a `weak' regularization. 
%However, such regularization, unlike $\ell_2$ norm, is non-convex.  
%We also show that learning with an adversary guarantees the uniqueness of the solution of logistic regression on small-margin linear separable data.
%Note that on linear separable data, logistic regression (without regularization) has no valid solution since the weight matrix $W$ is approaching infinity.
%This observation provides another interpretation of its regularizing behavior. 

The main achievement of this paper is a training method that 
is able to produce robust classifiers with high classification accuracy 
in the face of stronger data perturbation.
The approach we propose differs from previous approaches in a few ways.
First,
\citet{goodfellow2014explaining} suggests using an augmented objective
that combines the original training objective with an additional objective 
that is measured after after the training inputs have been perturbed.
Alternatively, \citep{nokland2015improving} suggest, 
as a specialization of the method in \citep{goodfellow2014explaining}, 
to only use the objective defined on the perturbed data. 
However, there is no theoretical analysis to justify that 
classifiers learned in this way are indeed robust;
both methods are proposed heuristically.
In our proposed approach,
we formulate the learning procedure as a min-max problem that forces the 
learned DNN model to be robust against adversarial examples, so that the learned classifier is inherently robust.
In particular, we allow an adversary to apply perturbations to each 
data point in an attempt to maximize classification error,
while the learning procedure attempts to minimize misclassification error 
against the adversary.
We call this learning procedure `learning with a strong adversary'.
Such min-max formulation has been discussed specifically in \citep{goodfellow2014explaining}, but emphasis is on its regularization effect particularly for logistic regression.
Our setting is more general and applicable to different loss functions and different types of perturbations, and is the origin of our learning procedure.\footnote{ Analysis of the regularization effect of such min-max formulation is postponed in the appendix, since it is not closly related to the main content of the paper.}
It turns out that an efficient method for finding such adversarial examples 
is required as an intermediate step to solve such a min-max problem,
which is the first problem we address.
Then we develop the full min-max training procedure that incorporates
robustness to adversarially perturbed training data.
The learning procedure that results turns out to have some similarities
to the one proposed in \citep{nokland2015improving}.
Another min-max formulation is proposed in \citet{miyato2015distributional} but still with the interpretation of regularization. 
These approaches are based on significantly 
%are conduct totally independently from 
different understandings of this problem.
Recently, a theoretical exploration of the robustness of classifiers 
\citep{fawzi2015analysis} suggests that, as expected, 
there is a trade-off between expressive power and robustness. 
This paper can be considered as an exploration into this same
trade-off from an engineering perspective.

%Our learning procedure is proposed based on a min-max formulation, while in \citep{nokland2015improving} such learning procedure is proposed heuristically.

\if0 
Lastly, \citep{miyato2015distributional} and \citep{nokland2015improving} also use the adversarial examples to help the training of the DNN networks.
Improvements on the MNIST data set are reported using adversarial training. 
Such learning procedure is also interpreted as a better regularization compared to dropout in \citep{nokland2015improving}.
We also further investigate such learning procedure on the binary logistic regression problem. 
We show three interesting properties of 
learning with a strong adversary 
for logistic regression, 
which provides deeper understanding of its regularizing behavior. 
We test our algorithm on both MNIST and CIFAR-10. 
Experimental results do not suggest superiority of 
learning with a strong adversary 
versus dropout. 
However, we succeed to learn a much more robust network that is slightly worse than, if not fairly the same as, the one trained with dropout in misclassification rate. 
\fi 

The remainder of the paper is organized as follows.
First, we propose a new method for finding adversarial examples 
in Section \ref{sec:FindingAdExam}. 
Section \ref{sec:Robustness} is then devoted to developing
the main method:
a new procedure for learning with a stronger form of adversary. 
Finally, we provide an experimental evaluation 
of the proposed method
on MNIST and CIFAR-10 in 
Section \ref{sec:Experiments}.    

%One concern of the proposed method may be about the trade-off between the model's approximating power and its robustness.
%Due to the highly expressive power of DNN models, we are expecting the network to still have good performance as well as 

\subsection{Notations}

We denote the supervised training data by 
$\uZ = \{(x_1, y_1),\ldots, (x_N,y_N) \}$. 
Let $K$ be the number of classes in the classification problem.  
The loss function used for training will be denoted by $\ell$. 
Given a norm $\|\cdot\|$, let $\|\cdot\|_*$ denote its dual norm,
such that $\|u\|_* = \max_{\|v\| \le 1} \langle u,\, v\rangle$. 
Denote the network by $\cN$ whose last layer is a softmax layer 
$g(x) \triangleq \alpha = (\alpha_1, \ldots, \alpha_K)$
to be used for classification. So $\cN(x)$ is the predicted label for the sample $x$.
%The prediction of $\cN$ is denoted by $\hat{y}$.

%% Please note that we have introduced automatic line number generation
%% into the style file for \LaTeXe. This is to help reviewers
%% refer to specific lines of the paper when they make their comments. Please do
%% NOT refer to these line numbers in your paper as they will be removed from the
%% style file for the final version of accepted papers.

\section{Finding adversarial examples}
\label{sec:FindingAdExam}

%Consider a network $\cN$ that uses softmax as its last layer for classification.
Consider an example $(x,y) \in \cX\times \{1,2,\ldots, K\}$ 
and assume that $\cN(x)=y$, 
where $y$ is the true label for $x$.
Our goal is to find a small perturbation $r\in \cX$ so that $\cN(X+r)\neq y$. 
This problem was originally investigated by \citet{szegedy2013intriguing}, 
who propose the following perturbation procedure: given $x$, solve
\begin{align*}
\min_r \|r\|  \quad
s.t. \quad \cN(X+r)\neq \cN(X).
\end{align*}
The simple method we propose to find such a perturbation $r$ is based on the 
linear approximation of $g(x)$, $\hat{g}(x+r) = g(x) + Hr$, 
where $H = \frac{\partial g}{\partial w}\vert_x$ is the Jacobian matrix.

As an alternative, we consider the following question: 
for a fixed index $j\neq y$, 
what is the minimal $r_{(j)}$ satisfying $\cN(x+r_{(j)}) = j$? 
Replacing $g$ by its linear approximation $\hat{g}$, 
one of the necessary conditions for such a perturbation $r$ is:
\[
H_jr_{(j)} - H_yr_{(j)} \geq \alpha_y - \alpha_j,
\]
where $H_j$ is the $j$-th row of $H$.  
Therefore, the norm of the optimal $r_{(j)}^*$ is greater than the following objective value:
\begin{align*} 
\min_{r_{(j)}} \|r_{(j)}\| \numberthis \label{eq:duelnorm} \quad 
s.t. \quad H_jr_{(j)} - H_yr_{(j)}\geq \alpha_y - \alpha_j.
\end{align*}
The optimal solution to this problem is provided in 
Proposition \ref{prop:duelnorm}.

\begin{prop}
\label{prop:duelnorm}
It is straightforward that the optimal objective value is 
$\|r_{(j)}\| = \frac{\alpha_y - \alpha_j}{\|H_j - H_y\|_*}$.
In particular,
the optimal $r_{(j)}^*$ for common norms are:
\begin{enumerate}
\item 
If $\|\cdot\|$ is the $L_2$ norm,
then $r_{(j)}^* = \frac{\alpha_y - \alpha_j}{\|H_j - H_y\|_2^2}(H_j - H_y)$;
\item 
If $\|\cdot\|$ is the $L_{\infty}$ norm, 
then $r_{(j)}^* = \frac{\alpha_y - \alpha_j}{\|H_j - H_y\|_1} \text{sign}(H_j - H_y)$;
\item 
If $\|\cdot\|$ is the $L_1$ norm, 
then $r_{(j)}^* = \frac{c}{\|H_j - H_y\|_{\infty}}e_k$ 
where $k$ satisfies $|(H_j - H_y)_k| = \|H_j - H_y\|_{\infty}$. 
Here $V_k$ is the $k$-th element of $V$. 
\end{enumerate}
\end{prop}

However, such $r_{(j)}^*$ is necessary 
but not sufficient to guarantee that $\arg\!\max_i \hat{g}(x+r_{(j)})_i = j$. 
The following proposition shows that in order to have $\hat{g}$ make 
a wrong prediction, 
it is enough to use the minimum among all $r_{(j)}^*$'s. 

\begin{prop}
\label{prop:AdvR}
Let $I = \arg\!\min_i \|r_{(i)}^*\|$. Then $r_I^*$ is the solution of the following problem:
\begin{align*}
 \min_r \, \|r\|  \quad 
s.t. \quad \arg\!\max_i \, \left(\hat{g}(X+r)\right)_i \neq y.
\end{align*} 
\end{prop}
\vspace{-5mm}
Putting these observations together, we achieve an algorithm 
for finding adversarial examples, as shown in Algorithm \ref{alg:FindingAdv}.
\begin{algorithm} [h]
\caption{Finding Adversarial Examples}
\label{alg:FindingAdv}
\begin{algorithmic}[1]
\INPUT $(x, y)$; Network $\cN$; 
\OUTPUT $r$
\STATE Compute $H$ by performing forward-backward propagation from the input layer to the softmax layer $g(x)$
\FOR {$j = 1,2,\ldots, K$}
\STATE Compute $r_{(j)}^*$ from Equation \eqref{eq:duelnorm}
\ENDFOR
\STATE Return $r = r_{(I)}^*$ where $I = \arg\!\min_i \|r_{(i)}^*\|$.
\end{algorithmic}
\end{algorithm}

\section{Toward Robust Neural Networks}
\label{sec:Robustness}

We enhance the robustness of a neural network model
by preparing the network for the worst examples by training with the following
objective:

\begin{equation}
\label{eq:LWA}
\min_g \sum_i \max_{\|r^{(i)}\|\leq c} \ell(g(x_i+r^{(i)}), y_i);
\end{equation}

\noindent
where $g$ ranges over functions expressible by the network model $\cN$
(i.e.\ ranging over all parameters in $\cN$).
In this formulation,
the hyperparameter $c$ that controls 
the magnitude of the perturbation needs to be tuned.
Note that when 
$\ell(g(x_i+r^{(i)}), y_i) = \mI_{(\max_j (g(x_i+r^{(i)})_j) \neq y_i)}$, 
the objective function is the misclassification error under perturbations. 
Often, $\ell$ is a surrogate for the misclassification loss 
that is differentiable and smooth.
Let $L_i(g) = \max_{\|r^{(i)}\|_2\leq c}\ell(g(x_i+r^{(i)}), y_i)$. 
Thus the problem is to find $g^* = \arg \min_g \sum_i L_i(g)$. 

{\bf Related Works:} \\
1. When taking $\ell$ to be the logistic loss, $g$ to be a linear function, and the norm for perturbation to be $\ell_{\infty}$, then the inner max problem has analytical solution, and Equation \eqref{eq:LWA} matches the learning objective in Section 5 of \citep{goodfellow2014explaining}. \\
2. Let $\ell$ be the negative log function and $g$ be the probability predicted by the model. Viewing $y_i$ as a distribution that has weight 1 on $y_i$ and $0$ on other classes, then the classical entropy loss is in fact $D_{KL}(y_i \| p)$ where $p = g(x_i)$. Using these notions, Equation \eqref{eq:LWA} can be rewritten as $\min_f \sum \max_{r^{(i)}} D_{KL}(y_i\, \| \tilde{p})$ where $\tilde{p} = g(x_i + r^{(i)})$. Similarly, the objective function proposed in \citep{goodfellow2014explaining} can be generalized as $\alpha D_{KL}(y_i \| p) + (1-\alpha) \max_{r^{(i)}} D_{KL}(y_i \| \tilde{p})$. Moreover, the one proposed in \citet{miyato2015distributional} can be interpreated as $D_{KL}(y_i \| p) + D_{KL}(p \| \tilde{p}) $. Experiments shows that our approach is able to achieve the best robustness while maintain high classification accuracy.

To solve the problem \eqref{eq:LWA} using SGD, one needs to compute the 
derivative of $L_i$ with respect to (the parameters that define) $g$. 
The following preliminary proposition 
suggests a way of computing this derivative. 

\begin{prop}
\label{prop:dL}
Given $h: \cU\times \cV \rightarrow \cW$ differentiable almost everywhere, 
define $L(v) = \max_{u\in\cU} h(u,v)$. 
Assume that $L$ is uniformly Lipschitz-continuous as a function of $v$, 
then the following results holds almost everywhere:
\[
\frac{\partial L}{\partial v}(v_0) = \frac{\partial h}{\partial v}(u^*, v_0),
\]
where $u^* = \arg\max_u h(u,v_0)$. 
\end{prop}

\begin{proof}
Note that $L$ is uniformly Lipschitz-continuous, 
therefore by Rademacher's theorem, $L$ is differentiable almost everywhere. 
For $v_0$ where $L$ is differentiable, the Fr\'{e}chet subderivative of $L$ is actually a singleton set of its derivative.

Consider the function $\hat{L}(v) = h(u^*,v)$. 
Since $h$ is differentiable, $\frac{\partial h}{\partial v}(u^*, v_0)$ 
is the derivative of $\hat{L}$ at point $v_0$. 
Also $\hat{L}(v_0) = L(v_0)$. 
Thus, by Proposition 2 of \citep{neu2012apprenticeship}, 
$\frac{\partial h}{\partial v}(u^*, v_0)$ also belongs to the 
subderivative of $L$. 
Therefore, 
\[
\frac{\partial L}{\partial v}(v_0) = \frac{\partial h}{\partial v}(u^*, v_0).
\]
\end{proof}
\vspace{-5mm}
The differentiability of $h$ in Proposition \ref{prop:dL} usually holds. 
The uniformly Lipschitz-continuous of neural networks was also discussed 
in the paper of \citet{szegedy2013intriguing}. 
It still remains to compute $u^*$ in Proposition \ref{prop:dL}. 
In particular given $(x_i, y_i)$,
we need to solve
\begin{equation}
\label{eq:advexam}
\max_{\|r^{(i)}\|\leq c} \ell(g(x_i+r^{(i)}), y_i).
\end{equation}
We postpone the solution for the above problem to the end of this section.
Given that we can have an approximate solution for Equation \eqref{eq:advexam},
a simple SGD method to compute a local solution for Equation \eqref{eq:LWA} 
is then shown in Algorithm \ref{alg:LWA}.

\begin{algorithm} [h]
\caption{Learning with an Adversary}
\label{alg:LWA}
\begin{algorithmic}[1]
\INPUT $(x_i, y_i)$ for $1\leq i\leq N$; Initial $g_0$; 
\OUTPUT $\hat{g}$
\FOR{ $t = 1,2, \ldots, T$}
\FOR{$(x_i, y_i)$ in the current batch}
\STATE Use forward-backward propagation to compute $\frac{\partial g}{\partial x}$
\STATE Compute $r^*$ as the optimal perturbation to $x$, using the proposed methods in Section \ref{subsec:ComputeR}
\STATE Create a pseudo-sample to be $(\hat{x}_i = x_i + c \frac{r^*}{\|r^*\|_2}, y_i) $
\ENDFOR
\STATE Update the network $\hat{g}$ using forward-backward propagation on the pseudo-sample $(\hat{x}_i, y_i)$ for $1\leq i \leq N$
\ENDFOR 
\STATE Return $\hat{g}$.
\end{algorithmic}
\end{algorithm}

For complex prediction problems, deeper neural networks are usually proposed, 
which can be interpreted as consisting of two parts: 
the lower layers of the network can be interpreted as learning a 
representation for the input data, 
while the upper layers can be interpreted as learning a classification model
on top of the learned representation.
The number of layers that should be interpreted as providing representations 
versus classifications is not precise and varies between datasets;
we treat this as a hyperparameter in our method.
Given such an interpretation, 
denote the representation layers of the network as $\cN_{\rm rep}$ 
and the classification layers of the network as $\cN_{\rm cla}$. 
We propose to perform the perturbation over the output of $\cN_{\rm rep}$ 
rather than the raw data. 
Thus the problem of learning with an adversary can be formulated as follows:
\begin{equation}
\label{eq:LWA2}
\min_{\cN_{\rm rep}, \cN_{\rm cla}} \sum_i \max_{\|r^{(i)}\|\leq c} \ell\left(\cN_{\rm cla}\left(\cN_{\rm rep}(x_i)+r^{(i)}\right), y_i\right).
\end{equation}
Similarly, Equation \eqref{eq:LWA2} can be solved by the following SGD method 
given  in Algorithm \ref{alg:LWA2}. 

\begin{algorithm} [h]
\caption{Learning with an Adversary in the Split Network Interpretation}
\label{alg:LWA2}
\begin{algorithmic}[1]
\INPUT $(x_i, y_i)$ for $1\leq i\leq N$; Initial $\cN_{\rm cla}$ and $\cN_{\rm rep}$; 
\OUTPUT $\hat{f}$
\FOR{ $t = 1,2, \ldots, T$}
\FOR{$(x_i, y_i)$ in the current batch}
\STATE Use forward propagation to compute the output of $\cN_{\rm rep}$, $\tilde{x}_i$
\STATE Take $\tilde{x}_i$ as the input for $\cN_{\rm cla}$
\STATE Use forward-backward propagation to compute $\frac{\partial \alpha}{\partial \tilde{x}_i}$
\STATE Compute $r^*$ as the optimal perturbation to $\tilde{x}_i$, using the proposed methods in Section \ref{subsec:ComputeR}
\STATE Create a pseudo-sample to be $(\hat{\tilde{x}}_i = \tilde{x}_i + c \frac{r^*}{\|r^*\|}, y_i) $
\ENDFOR
\STATE Use forward propagation to compute the output of $\cN_{\rm cla}$ on $(\hat{\tilde{x}}_i, y_i)$ for $1\le i\le N$
\STATE Use backward propagation to update both $\cN_{\rm cla}$ 
\STATE Use backward propagation to compute  $\frac{\partial \cN_{\rm rep}}{\partial W}\vert_{\tilde{x_i}, y_i}$
\STATE Update  $\cN_{\rm rep}$ by $ \frac{\partial \ell}{\partial \tilde{x}_i}\vert_{(\hat{\tilde{x}}_i, y_i)}\frac{\partial \cN_{\rm rep}}{\partial W}\vert_{\tilde{x_i}, y_i}$
\ENDFOR 
\STATE Return $\cN_{\rm cla}$ and $\cN_{\rm rep}$
\end{algorithmic}
\end{algorithm}

\subsection{Computing the perturbation}
\label{subsec:ComputeR}
We propose two different perturbation methods based on two different 
principles. 
The first proposed method, 
similar to that of \citep{goodfellow2014explaining}, 
does not require the solution of an optimization problem. 
Experimental results show that this method, compared to the method proposed 
in \citep{goodfellow2014explaining}, is more effective in the sense that,
under the same magnitude of perturbation, 
the %performance of the network is worse.
accuracy reduction in the network is greater.

\subsubsection{Likelihood Based Loss}

Assume the loss function $\ell(x,y) = h(\alpha_y)$,
where $h$ is a non-negative decreasing function. 
A typical example of such a loss would be the logistic regression loss. 
In fact, most of the network models use a softmax layer as the last layer 
and a cross-entropy objective function. 
%All these networks can fit into this type of loss function.
Recall that we would like to find 
\begin{eqnarray}
r^* = \arg\max_{\|r^{(i)}\|\leq c} h\left(g(x_i+r^{(i)})_{y_i}\right),
\label{eq:rr}
\end{eqnarray}
where $x_i$ could be the raw data or the output of $\cN_{\rm rep}$.
Since $h$ is decreasing, 
$r^* = \arg\min_{\|r^{(i)}\|\leq c} g(x_i+r^{(i)})_{y_i}$.

This problem can still be difficult to solve in general. 
Similarly, using its linear approximation
 $\tilde{g}(x_i+r^{(i)})_{y_i}$, 
i.e.\
$g(x_i+r^{(i)})_{y_i} 
\approx 
\tilde{g}(x_i+r^{(i)})_{y_i} = g(x_i)_{y_i}+\langle H_{y_i},\,r^{(i)}\rangle$, 
such that $H = \frac{\partial g}{\partial w}\vert_x$ is the Jacobian matrix,
it is then easy to see that 
using $\tilde{g}(x_i+r^{(i)})_{y_i}$ 
in Equation \eqref{eq:rr},
the solution for $r^*$ can be recovered as
$r^* = \{r:\, \|r\|\le c;\, \langle H_{y_i},\,r^{(i)}\rangle 
= c\|H_{y_i}\|_*\}$.

The optimal solutions for $r^*$ given common norms are:
\begin{enumerate}
\item 
If $\|\cdot\|$ is the $L_2$ norm,
then $r_{(j)}^* = c\,\frac{H_{y_i}}{\|H_{y_i}\|_2}$;
\item 
If $\|\cdot\|$ is the $L_{\infty}$ norm,
then $r_{(j)}^* = c\, \text{sign}(H_{y_i})$;
\item 
If $\|\cdot\|$ is the $L_1$ norm,
then $r_{(j)}^* = c\,e_k$, where $k$ satisfies 
$|H_{y_i}| = \|H_{y_i}\|_{\infty}$. 
%Here $V_k$ is the $k$-th element of $V$. 
\end{enumerate}
Note that the second item here is exactly the method suggested in 
\citep{goodfellow2014explaining}. $\ell_2$ norm is used in \citep{miyato2015distributional} but with different objective function rather than negative log likelihood, as mentioned at the begining of this section.

\subsubsection{Misclassification Based Loss}

In the case that the loss function $\ell$ is a surrogate loss for the 
misclassification rate, it is reasonable to still use the misclassification 
rate as the loss function $\ell$ in Equation \eqref{eq:advexam}.
In this case, the problem in Equation \eqref{eq:advexam} becomes finding a 
perturbation $r: \|r\|\le c$ that forces the network $\cN$ to misclassify $x_i$.
In practice, for $\cN$ to achieve a good approximation, 
$c$ needs to be chosen to have a small value, 
hence it might not be large enough to force a misclassification.
One intuitive way 
to achieve this 
is to perturb by $r$ in the direction proposed in 
Section \ref{sec:FindingAdExam}, 
since such direction is arguably the most damaging direction for 
the perturbation. 
Therefore, 
in this case, we use
\[
r^* = c\, r_I^*/\|r_I^*\|,
\]
where $r_I^*$ is the output of Algorithm \ref{alg:FindingAdv}.
 
\section{Experimental Evaluation}
\label{sec:Experiments}

To investigate the training method proposed above, in conjunction with the
different approaches for determining perturbations,
we consider the MNIST \citep{lecun1998mnist} and CIFAR-10 data sets.
The MNIST data set contains 28x28 grey scale images of handwritten digits. 
We normalize the pixel values into the range [0, 1] by dividing by 256.
The CIFAR-10 \citep{krizhevsky2009learning} dataset is a tiny nature image dataset. CIFAR-10 datasets contains 10 different classes images, each image is an RGB image in size of 32x32. 
Input images are subtracted by mean value 117, and randomly cropped to size 28x28. We also normalize the pixel value into the range (-1, 1) by dividing 256. This normalization is for evaluating perturbation magnitude by $L_2$ norm. 
For both datasets, we randomly choose 50,000 images for training and 10,000 for testing. 

All experiment models are trained by using MXNet \citep{chen2015mxnet}\footnote{Reproduce code: \url{https://github.com/Armstring/LearningwithaStrongAdversary}}. % cite is not ready

\subsection{Finding adversarial examples}

We tested a variety of different perturbation methods on MNIST, including: 
1. Perturbation based on $\alpha$ using the $\ell_2$ norm constraint, 
as shown in Section \ref{sec:FindingAdExam}  (Adv\_Alpha); 
2. Perturbation based on using a loss function with the $\ell_2$ norm
constraint,
as shown in Section \ref{subsec:ComputeR} (Adv\_Loss);
3. Perturbation based on using a loss function with the $\ell_\infty$ norm 
constraint,
as shown in Section \ref{subsec:ComputeR} (Adv\_Loss\_Sign).
In particular, a standard `LeNet' model is trained on MNIST, 
with training and validation accuracy being $100\%$ and $99.1\%$
respectively.
Based on the learned network, different validation sets are then generated 
by perturbing the original data with different perturbation methods. 
The magnitudes of the perturbations range from $0.0$ to $4.0$ in $\ell_2$ norm.
An example of such successful perturbation is show in Table \ref{tab:visualexample}. The classification accuracies on differently perturbed data sets are 
reported in Figure \ref{fig:perturbation}. 
\begin{table}[h]
	\centering
	\caption{An visual example of using different perturbations; The magnitude of all the perturbations are 1.5 in $\ell_2$ norm. The true label of the image is 8. The first perturbed image is still predicted to be 8, while the other 2 perturbed images are predicted to be 3. }
	\label{tab:visualexample}
	\begin{tabular}{|>{\centering\arraybackslash} m{2cm}|>{\centering\arraybackslash} m{2.5cm}|>{\centering\arraybackslash} m{2cm}|>{\centering\arraybackslash} m{2cm}|c |}
		\hline 
		& {\bf Adv\_Loss\_Sign } & {\bf Adv\_Loss} & {\bf Adv\_Alpha} &  {\bf Original Image}  \\  \hline 
		{\bf Perturbed Image} & 
		\includegraphics[width=0.13\columnwidth]{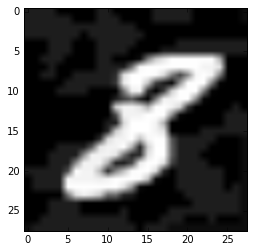}	&
		\includegraphics[width=0.13\columnwidth]{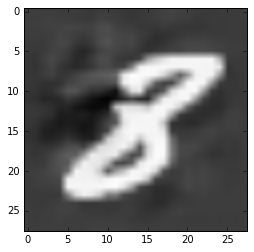} &
		\includegraphics[width=0.13\columnwidth]{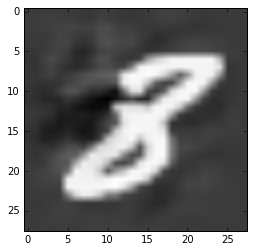}  & 
		\multirow{2}{*}{\includegraphics[width=0.13\columnwidth]{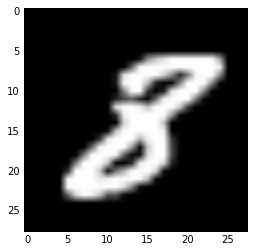}} \\ \cline{1-4} 
		{\bf Noise} &
		\includegraphics[width=0.13\columnwidth]{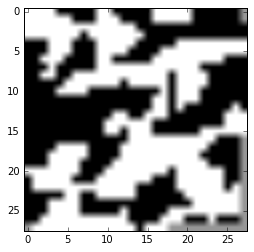}	&
		\includegraphics[width=0.13\columnwidth]{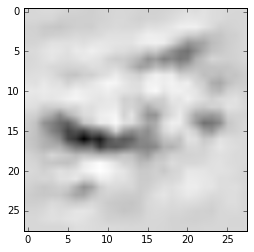} &
		\includegraphics[width=0.13\columnwidth]{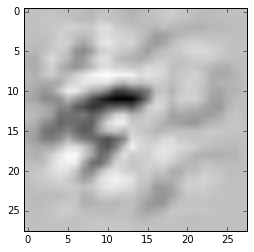}  & \\
		\hline         
	\end{tabular}
\end{table}

\begin{figure}[h]
\begin{center}
	\includegraphics[width = 0.7\textwidth]{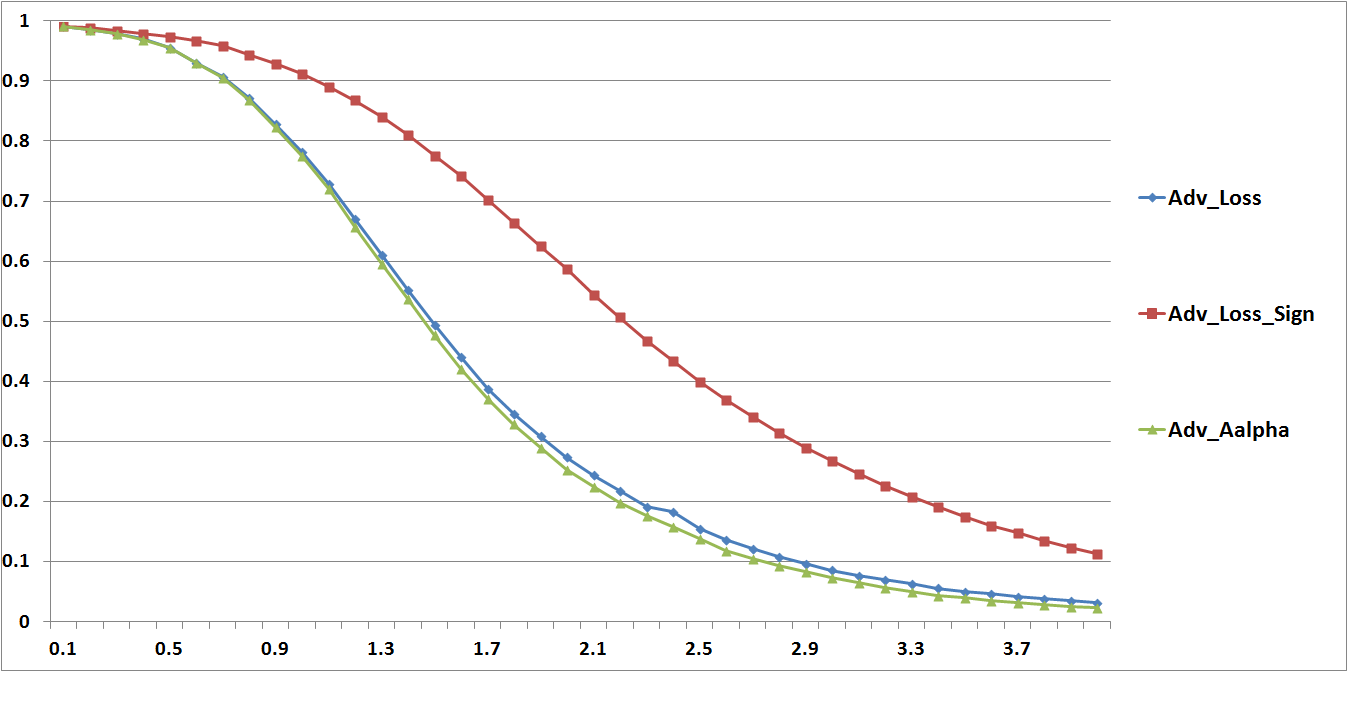}
\end{center}
\caption{Validation accuracies of different perturbation methods. $x$ axis denotes the $\ell_2$ norm of the perturbation.}
\label{fig:perturbation}
\end{figure}

The networks classification accuracy decreases with increasing
magnitude of the perturbation. 
These results suggest that Adv\_Alpha is consistently, but slightly, 
more effective than Adv\_Loss, 
and these two method are significantly more effective than Adv\_Loss\_Sign.

\if0
\emph{\bf Is it reasonable to use $\ell_2$ norm to measure the magnitude?} 
\\
One concern is that the following case might happen: 
a small perturbation in $\ell_2$ norm places most of its weight on a specific 
position, thus change the input image distinguishably. 
For example, given a mnist image of number `2', if the perturbation focuses on the tail, it may change the number to `7' under a perturbation of small $\ell_2$ norm. (but this perturabation has large $\ell_{\infty}$ norm.)
However, we don't observe such a phenomenon in our experiments. 
\fi

\emph{\bf Drawback of using $\alpha$ to find perturbations} 
\\
Note that the difference in perturbation effectiveness between using 
$\alpha$ and using the loss function is small. 
On the other hand, to compute the perturbation using $\alpha$, 
one needs to compute $\frac{\partial \alpha}{\partial x}$, 
which is $K$ times more expensive than the method using the loss function,
which only needs to compute $\frac{\partial \ell}{\partial x}$. 
(Recall that $K$ is the number of classes.) Due to time limit, during the training procedure we only use the adversarial examples generated from the loss function for our experiments.

\subsection{Learning with an adversary}

We tested our overall learning approach on both MNIST and CIFAR-10.
%In these previous studies,  the perturbed data used to evaluate the robustness of the learned classifier is generated from a network trained with clean data.
%It is more convincing that the learned classifier is robust in \citet{goodfellow2014explaining} where  the result on the data generated from the same classifier are also reported, but comparatively limited.
We measure the robustness of each classifier on various adversarial sets.
An adversarial set of the same type for each learned classifier is generated 
based on the targeted classifier.
We generated 3 types of adversarial data sets for the above 5 classifiers 
corresponding to Adv\_Alpha, Adv\_Loss, and Adv\_Loss\_Sign. 
In addition, we also evaluated the accuracy of these 5 classifiers on a 
fixed adversarial set, which is generated based on the `Normal' network 
using Adv\_Loss. 
Finally, we also report the original validation accuracies of 
the different networks. 
We use $\ell_2$ norm for our method to train the network.

\emph{\bf Experiments on MNIST:} 
\\
We first test different training methods on a 2-hidden-layer neural network 
model.
% that has 1000 hidden nodes in each hidden layer. 
In particular, we considered:
1. Normal back-forward propagation training, $100 \times 100$ (Normal); 
2. Normal back-forward propagation training with Dropout, $200 \times 200$ (Dropout \citep{hinton2012improving}); 
3. The method in \citep{goodfellow2014explaining}, $500 \times 500$ (Goodfellow's method); 
4. Learning with an adversary on raw data, $500 \times 500$ (LWA); 
5. Learning with an adversary at the representation layer,  $200 \times 500$ (LWA\_Rep).
Here $n\times m$ denote that the network has $n$ nodes for the first hidden layer, and $m$ nodes for the second one.
All of the results are tested under perturbations of with the
$\ell_2$ norm constrained to at most $1.5$.

\begin{table}[h]
\caption{Classification accuracies for 2-hidden-layers neural network on MNIST: the best performance on each adversarial sets are shown in bold. The magnitude of perturbations are $1.5$ in $\ell_2$ norm.}
\label{tab:LWADNN}
\begin{center}
\begin{tabular}{lccccc}
\multicolumn{1}{c}{\bf METHODS}  & & & \multicolumn{1}{c}{\bf Validation Sets} \\
       & Validation & Fixed & Adv\_Loss\_Sign & Adv\_Loss & Adv\_Alpha
        \\ \hline \\
Normal   		 & 0.977 	& 0.361 	& 0.287 	& 0.201 	& 0.074 \\
Dropout         	 & 0.982 	& 0.446 	& 0.440 	& 0.321 	& 0.193 \\
Goodfellow's Method   & {\bf0.991}& 0.972 	& 0.939 	& 0.844 	& 0.836 \\
LWA         		 & {\bf0.990}& {\bf0.974}& {\bf0.944}& {\bf0.867}& {\bf0.862} \\
LWA\_Rep         & 0.986 	& 0.797		& 0.819 	& 0.673 	& 0.646 
\end{tabular}
\end{center}
\end{table}

We sumarize the results in Table \ref{tab:LWADNN}. 
Note that the normal method can not afford any perturbation on the validation 
set, showing that it is highly non-robust.
By training with dropout, both the accuracy and robustness of the 
neural network are improved, but robustness remains weak;
for the adversarial set generated by Adv\_Alpha in particular,
the resulting classification accuracy is only $19.3\%$.
Goodfellow's method improves the network's robustness greatly, 
compared to the previous methods. 
However, the best accuracy and the most robustness are both achieved by LWA. 
In particular, on the adversarial sets generated by our methods 
(Adv\_Loss and Adv\_Alpha), the performance is improved from 
$84.4\%$ to $86.7\%$, and from $83.6\%$ to $86.2\%$.
The result of LWA\_Rep is also reported for comparison. 
Overall, it achieves worse performance than Goodfellow's method 
\citep{goodfellow2014explaining} and LWA, but still much more robust than Dropout.  

We also evaluated these learning methods on the LeNet model 
\citep{lecun1998gradient}, which is more complex, including convolution layers. 
We use Dropout for Goodfellow's method and LWA.
The resulting learning curves are reported in Figure \ref{fig:LWALenet}. 
It is interesting that we do not observe the trade-off between robustness 
and accuracy once again;
this phenomenon also occurred with the 2-hidden-layers neural network. 
\begin{figure}[h]
\begin{center}
	\includegraphics[width = 0.7\textwidth]{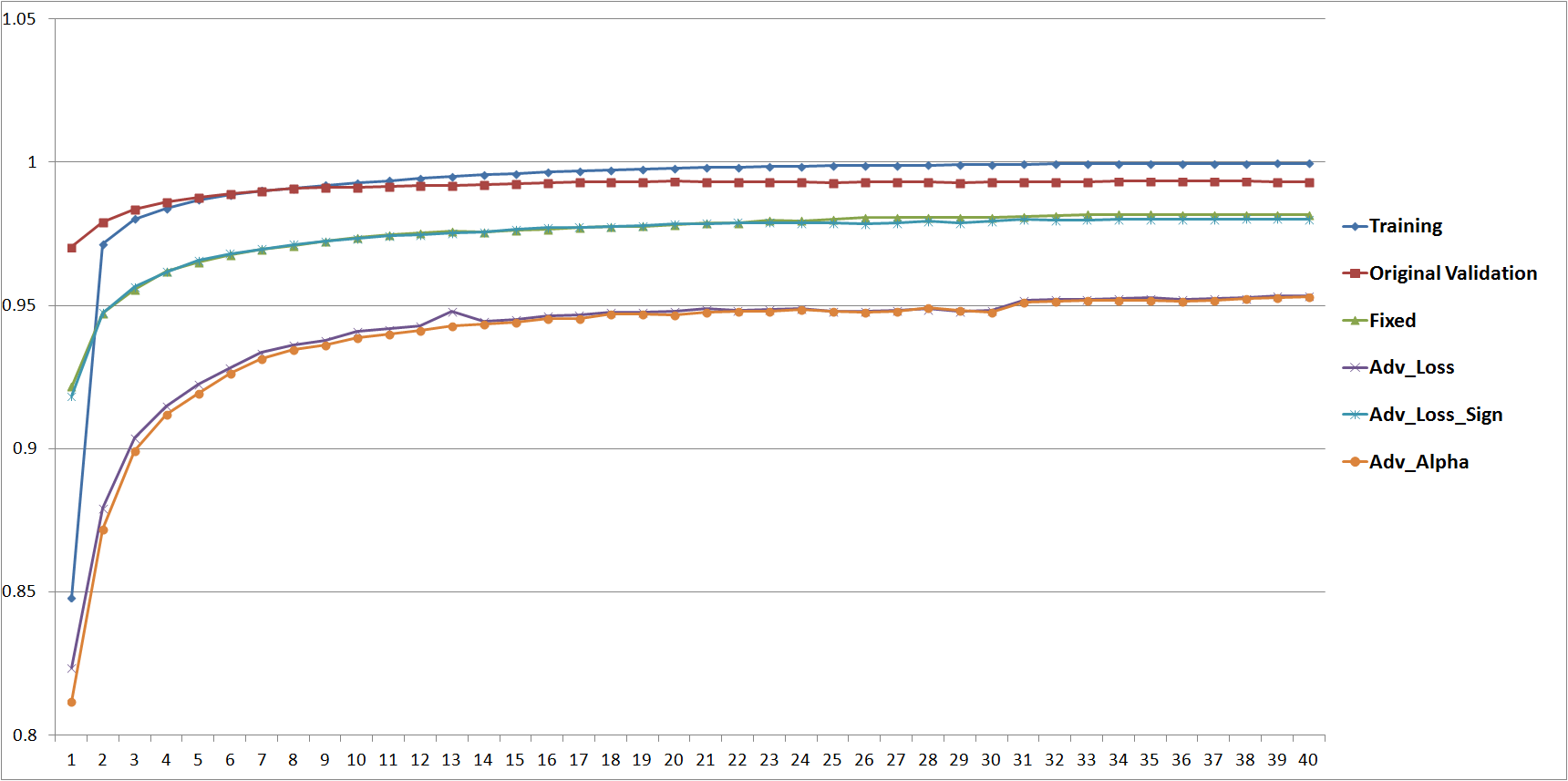}
\end{center}
\caption{Validation accuracies on MNIST for different perturbation methods. $x$ axis denotes the $\ell_2$ norm of the perturbation.}
\label{fig:LWALenet}
\end{figure}
The final result is summarized in Table \ref{tab:LWALenet}, which shows the great robustness of LWA. We don't observe the superiority of perturbing the representation layer to perturbing the raw data. In the learned networks, a small perturbation on the raw data incurs a much larger perturbation on the representation layer, which makes LWA\_Rep difficult to achieve both high accuracy and robustness. How to avoid such perturbation explosion in the representation network remains open for future investigation.
\vspace{-5mm}
\begin{table}[h]
\caption{Classification accuracies for LeNet trained using LWA on MNIST. The magnitude of perturbations are $1.5$ in $\ell_2$ norm.}
\label{tab:LWALenet}
\begin{center}
\begin{tabular}{lccccc}
\multicolumn{1}{c}{\bf METHODS}  & & & \multicolumn{1}{c}{\bf Validation Sets} \\
       & Validation & Fixed & Adv\_Loss\_Sign & Adv\_Loss & Adv\_Alpha
        \\ \hline \\
Normal      		 & 0.9912 	& 0.7456 	& 0.8082 	& 0.5194 	& 0.5014 \\
Dropout      		 & 0.9922 	& 0.7217 	& 0.7694 	& 0.5322 	& 0.4893 \\
Goodfellow's method	 & {\bf 0.9937} 	& 0.9744 	& 0.9755 	&  0.9066	& 0.9035 \\
LWA           		 & {\bf 0.9934}	& {\bf0.9822}	& {\bf0.9859}	& {\bf0.9632}	& {\bf0.9627} 
\end{tabular}
\end{center}
\end{table}
\vspace{-5mm}

\emph{\bf Experiments on CIFAR-10:} \\
CIFAR-10 is a more difficult task compared to MNIST. 
Inspired by VGG-D Network \citep{simonyan2014very}, we use a network formed by 6 convolution layers with 3 fully connected layers. 
Same to VGG-Network, we use ReLU as the activation function. 
For each convolution stage, we use three 3x3 convolution layers followed by a max-pooling layer with stride of 2. We use 128, 256 filters for each convolution stage correspondingly. For three fully connected layer, we use 2048, 2048, 10 hidden units. We also split this network in representation learner and classifier view: the last two fully connected layers with hidden units 2048 and 10 are classifier and other layers below formed a representation learner. 
We compare the following methods:
1. Normal training (Normal); 
2. Normal training with Dropout (Dropout); 
3. Goodfellow's method with Dropout (Goodfellow's method);
3. Learning with a strong adversary on raw data with Dropout (LWA); 
4. Learning with a strong adversary at the representation layer with Dropout (LWA\_Rep). We use batch normalization (BN) \citep{ioffe2015batch} to stabilize the learning of LWA\_Rep.
We also test the performances of different methods with batch normalization (BN).
The results are summarized in Table \ref{tab:LWACIFAR}. Learning with a strong adversary again achieves better robustness, but we also observe a small decrease on their classification accuracies. 

\begin{table}[h]
	\vspace{-3mm}
	\caption{Classification accuracies on CIFAR-10: the best performance on each adversarial sets are shown in bold. The magnitude of perturbations are $0.5$ in $\ell_2$ norm.}
	\label{tab:LWACIFAR}
	\begin{center}
		\begin{tabular}{lccccc}
			\multicolumn{1}{c}{\bf METHODS}  & & & \multicolumn{1}{c}{\bf Validation Sets} \\
			& Validation & Fixed & Adv\_Loss\_Sign & Adv\_Loss & Adv\_Alpha
			\\ \hline \\
			Normal   		 & 0.850 	& 0.769 	& 0.510 	&  0.420	& 0.410 \\
			Dropout         	 &  {\bf 0.881} 	& 0.774 	& 0.571 	&  0.480	& 0.466 \\
			Goodfellow's method	 &  0.856 	& 0.818 	& 0.794 	&  0.745	& 0.734 \\
			LWA   		 & 0.864 & {\bf 0.831} & {\bf0.803}  &{\bf 0.760} & {\bf0.750} \\
			\hline
			Normal + BN  		 & 0.887 	& 0.793 	& 0.765 	&  0.747	& 0.738 \\
			Dropout + BN        	 &  {\bf 0.895} 	& 0.810 	& 0.729 	&  0.697	& 0.687 \\
			Goodfellow's method	+ BN &  0.886 	& 0.856 	& 0.821 	&  0.775	& 0.754 \\
			LWA + BN  		 & 0.890 & {\bf 0.863} & {\bf0.823}  &{\bf 0.785} & {\bf0.759} \\
			LWA\_Rep + BN         & 0.832 	& 0.746		&  0.636 	& 0.604 	& 0.574 
		\end{tabular}
	\end{center}
		\vspace{-2mm}
\end{table}

\section{Conclusion}
We investigate the curious phenomenon of 'adversarial perturbation' in a formal min-max problem setting. A generic algorithm is developed based on the proposed min-max formulation, which is more general and allows to replace previous heuristic algorithms with formally derived ones. We also propose a more efficient way in finding adversarial examples for a given network. The experimental results suggests that learning with a strong adversary is promising in the sense that compared to the benchmarks in the literature, it achieves significantly better robustness while maintain high normal accuracy.

\subsubsection*{Acknowledgments}
We thank Ian Goodfellow, Naiyan Wang, Yifan Wu for meaningful discussions. Also we thank Mu Li for granting us access of his GPU machine to run experiments.
This work was supported by the Alberta Innovates Technology Futures and NSERC.

\newpage
\bibliography{LWA}
\bibliographystyle{iclr2016_conference}

\include{appendix}
\end{document}

%% file: appendix.tex
%!TEX root =  paper.tex
\appendix
\section{Analysis of the Regularization Effect of Learning with an Adversary on Logistic Regression}
Perhaps the most simple neural network is the binary logistic regression. 
We investigate the behavior of learning with an adversary on the logistic regression model in this section. 
Let $\ell(z) = -\log (1+\exp (-z))$.
Thus logistic regression is to solve the following problem:
\begin{equation}
\label{eq:logisticregressionWA}
\min_{w} \sum_{i} \max_{\|r^{(i)}\|\le c} \ell(y_iw^{\top}(x_i + r^{(i)})).
\end{equation}

\begin{prop}
	The problem \ref{eq:logisticregressionWA} is still a convex problem.
\end{prop}
\begin{proof}
	In fact, there is a closed form solution for the maximization subproblem in \eqref{eq:logisticregressionWA}. Not that $\ell$ is strictly decreasing, thus, Equation \eqref{eq:logisticregressionWA} is equivalent to 
	\[
	\min_{w} \sum_{i} \ell(y_iw^{\top}x_i- c\|w\|_*).
	\]
	Let $h(w)$ denote the concave function of $w$, $y_iw^{\top}x_i- c\|w\|_*$. 
	To prove $\ell(h(w))$ is convex, consider
	\[
	\ell\left(h(w_1)) + \ell(h(w_2)\right) \,\ge 2\ell\left( \frac{h(w_1)+ h(w_2)}{2}\right) \, \ge 2\ell\left( h\left( \frac{w_1+w_2}{2}\right)\right),
	\]
	where the first inequality is because of the convexity of $\ell$, and the second inequality is because of the monotonicity of $\ell$ and the concavity of $h$.
\end{proof}

Let $R_z(w) = \max_{\|r^{(i)}\|\le c} \ell(y_iw^{\top}(x_i + r^{(i)})) - \ell(y_iw^{\top}x_i) \ge 0$. The behavior of $R_z$ is a data-dependent regularization. 
The next proposition shows that different to common $\ell_p$ regularization, it is NOT a convex function.    
\begin{prop}
	The induced regularization $R_z(w)$ is non-convex.
\end{prop}
\begin{proof}
	An counter example is enough to prove that $R_z(w)$ is non-convex. Let $c = 0.5$ and $z = (x,y) = (1,1)$, then $R_z(w) : \real \rightarrow \real$ is
	\[
	R_z(w)  = \log \left( 1+\exp (-0.5w)\right) - \log \left( 1+\exp (-w)\right). 
	\]
	%The figure of this function is as follows, which clearly indicates its non-convexity.
	The figure of this function clearly shows its non-convexity.
\end{proof} 

One of the common problem about logistic regression without regularization is that there is no valid solution given a linearly separable data set $\uZ$. 
The following proposition shows that this issue is relieved by learning with an adversary.
\begin{prop}
	Assume that the linearly separable data set $\uZ$ has a margin less than $2c$, then logistic regression with adversary is guaranteed to have bounded solutions. 
	Moreover, if $\|\cdot\|_*$ is twice differentiable almost everywhere and $\sum_i\frac{\partial Q_i}{\partial w} \frac{\partial Q_i}{\partial w}^{\top}$ is positive definite where $Q_i = -y_ix_i^{\top}w + \|w\|_*$, then logistic regression with adversary has unique solution.
\end{prop}
\begin{proof}
	Since the margin of $\uZ$ is less than $2c$, after the data is perturbed, it is no longer linear separable. 
	Therefore for any $w$, there exists some $i$ such that $-y_iw^{\top}x_i + c\|w\|_*$. Note that $-y_iw^{\top}x_i + c\|w\|_*$ is homogeneous in $w$. Thus, if the optimal solution $w^*$ has a infinity norm, then 
	\[
	\sum_i \log\left(1+ \exp(-y_iw^{*\top}x_i + c\|w^*\|_*) \right) = \infty.
	\]
	However, assigning $w=0$ have finite loss. Therefore, $w^*$ is bounded.
	
	Moreover, since $\|\cdot\|_*$ is twice differentiable, consider the Hession matrix
	\[
	H = \sum_i \frac{e^{Q_i}}{(1+e^{Q_i})^2} \frac{\partial Q_i}{\partial w} \frac{\partial Q_i}{\partial w}^{\top} + \frac{e^{Q_i}}{(1+e^{Q_i})^2}\frac{\partial^2Q_i}{\partial^2w}.
	\] 
	Note that $ \|w\|_*$ is convex, so is $Q_i$. Thus $\frac{\partial^2Q_i}{\partial^2w}$ is positive semi-definite. 
	Since $\sum_i\frac{\partial Q_i}{\partial w} \frac{\partial Q_i}{\partial w}^{\top}$ is positive definite, and $\frac{e^{Q_i}}{(1+e^{Q_i})^2} >0$ for any $i$, $H$ is positive definite. 
	Therefore, the objective function is strictly convex with respect to $w$.
	Combining the boundedness of $w$ and being strictly convex of $\ell$ leads to the unique of $w^*$. 
\end{proof}
\begin{remark}
	Unlike adding a regularization term to the objective function, learning with an adversary can only guarantee unique solutions, when the data set is linearly separable with small margin. In such sense, learning with an adversary is a weaker regularization.
\end{remark}